\documentclass{article}

\usepackage[preprint]{neurips_2026} 

\usepackage[utf8]{inputenc} 
\usepackage[T1]{fontenc}    
\usepackage{hyperref}       
\usepackage{url}            
\usepackage{microtype}      
\usepackage{xcolor}         

\usepackage{amsmath, amssymb, enumitem}

\usepackage{amsthm} 
\newtheorem{theorem}{Theorem}
\newtheorem{lemma}[theorem]{Lemma}
\newtheorem{proposition}[theorem]{Proposition}
\newtheorem{corollary}[theorem]{Corollary}

\newtheorem{assumption}{Assumption}

\usepackage{bbm}

\usepackage{mathtools} 
\usepackage{float}
\usepackage{graphicx}
\usepackage{subcaption}
\usepackage{algorithm}
\usepackage{algpseudocode}

\DeclareMathOperator{\Var}{Var}
\DeclareMathOperator{\Cov}{Cov}
\DeclareMathOperator{\E}{\mathbb{E}}

\title{Sharper Regret Bounds for Time-Varying Gaussian Process Bandits with Constant Exploration}

\author{%
  Matthias Mandl \\
  Delft Institute of Applied Mathematics\\
  Delft University of Technology\\
  Mekelweg 4, 2628 CD, Delft \\
  \texttt{M.J.Mandl-2@tudelft.nl} \\
  \And
  Hanne Kekkonen \\
  Delft Institute of Applied Mathematics\\
  Delft University of Technology\\
  Mekelweg 4, 2628 CD, Delft \\
  \texttt{H.N.Kekkonen@tudelft.nl} \\
}

\begin{document}

\maketitle

\begin{abstract}
We study Bayesian optimization in a time-varying environment where the unknown reward function evolves according to a Gaussian process drift model. Existing GP-UCB analyses in this setting typically require the exploration parameter to grow with the horizon to maintain uniform confidence bounds. Using per-round local confidence events, we show that GP-UCB can instead be run with a constant exploration parameter and obtain an expected-regret bound whose coefficient depends on the drift rate. We also derive a sharper time-varying maximum-information-gain bound. For the squared exponential kernel, it yields $\tilde\gamma_T/T=\widetilde{\mathcal O}(\epsilon^{1/2})$ and expected average regret $\widetilde{\mathcal O}(\epsilon^{1/4})$ in the persistent-drift regime. The same constant-exploration analysis also yields realized-regret guarantees. Simulations support the predicted logarithmic dependence of the bound-suggested exploration parameter on $1/\epsilon$.
\end{abstract}

\section{Introduction}
\label{sec:intro}

Many sequential decision problems require optimizing an unknown objective through a limited number of costly evaluations. Examples include tuning the hyperparameters of machine learning models \citep{snoek2012practical}, adapting control policies in robotics \citep{calandra2016bayesian}, and selecting treatment strategies in clinical trials \citep{villar2015multi}. In these settings, each evaluation reveals only imperfect information about the objective, forcing a fundamental trade-off between exploration and exploitation.

Bayesian optimization (BO) provides a framework for sequentially optimizing an unknown objective by combining a probabilistic surrogate model with an acquisition rule. In the standard GP bandit formulation, the objective is a fixed function $f:\mathcal{S}\to\mathbb{R}$. At each round $t$, the learner selects $x_t\in\mathcal{S}$ and observes $y_t=f(x_t)+z_t$, where $z_t$ is observation noise. Gaussian processes (GPs) are commonly used as surrogate models because they provide posterior mean and uncertainty estimates in closed form.

When decisions affect performance throughout the learning process, performance is typically measured by cumulative regret, $R_T=\sum_{t=1}^T(\max_{x\in\mathcal{S}} f(x)-f(x_t))$. This criterion originates in the multi-armed bandit literature \citep{robbins1952some, lai1985asymptotically, auer2002finite}. Algorithms such as Gaussian Process Upper Confidence Bound (GP-UCB) \citep{srinivas2012information} and GP-Thompson sampling \citep{russo2014learning, chowdhury2017kernelized} extend bandit methods to continuous domains and achieve sublinear regret under suitable kernel regularity assumptions, with guarantees typically expressed through the maximum information gain \citep{srinivas2012information, vakili2021information, scarlett2017lower}.

However, many applications are inherently dynamic. System performance may drift due to environmental changes or evolving user behavior, necessitating a time-dependent reward model $f_t:\mathcal{S} \to \mathbb{R}$. When the objective changes over time, sublinear cumulative regret is generally unachievable. The relevant theoretical goal shifts to bounding the average regret, $R_T/T$, as a function of the environmental drift rate.

A principled probabilistic model for this setting was introduced by \citet{bogunovic2016time}. In their time-varying Gaussian process (TV-GP) model, the latent function evolves according to a Markov recursion driven by GP innovations, which induces a separable spatio-temporal kernel. Under this model, the learner updates a temporally discounted GP posterior and selects actions using a modified GP-UCB rule. Their analysis shows that near-optimal regret rates are achievable up to logarithmic factors. Subsequent work has studied related models of nonstationarity, including change-point structure \citep{brunzema2022event}, irregular evaluation times \citep{imamura2020time}, and bounded variation in RKHS norms \citep{deng2022weighted}. A separate frequentist line studies nonstationary kernelized bandits under sup-norm variation, including algorithms \citep{zhou2021no, hong2023optimization, iwazaki2025near} and lower bounds \citep{cai2025lower}. Other extensions include kernel-agnostic adaptation \citep{ziomek2024time} and stale-data removal in dynamic Bayesian optimization \citep{bardou2024too, bardou2025optimizing}.

A recurring theoretical and practical limitation in these analyses is the exploration schedule. Existing regret guarantees typically require the GP-UCB exploration parameter, $\beta_t$, to grow with time $t$. This leads to deteriorating performance on long time horizons and introduces schedule tuning parameters. In practice, such schedules are difficult to calibrate and induce unnecessary over-exploration.

This paper revisits the TV-GP setting of \citet{bogunovic2016time}, where the temporal evolution is governed by a drift parameter $\epsilon>0$, and shows that GP-UCB can be run with a constant exploration parameter while retaining rigorous performance guarantees. Our analysis uses local confidence events at each round rather than a single event that holds uniformly over the horizon. For any fixed drift level, Theorem~\ref{thrm:improved_regret} gives an expected regret bound with linear dependence on $T$ and a coefficient controlled by $\tilde\gamma_T/T$. In the persistent-drift regime, this yields a coefficient that decreases as the drift rate decreases. This differs from standard deterministic uniform-confidence GP-UCB analyses in static environments, where the confidence parameter generally depends on the horizon or time.

The analysis also gives a sharper dependence on the drift rate through the time-varying maximum information gain. If the static information gain satisfies $\gamma_N=\widetilde{\mathcal O}(N^q)$, then in the persistent-drift regime
\[
\frac{\tilde\gamma_T}{T}
=
\widetilde{\mathcal O}\!\left(\epsilon^{\frac{2(1-q)}{4-q}}\right),
\qquad
\frac{\mathbb E[R_T]}{T}
=
\widetilde{\mathcal O}\!\left(\epsilon^{\frac{1-q}{4-q}}\right).
\]
For the squared exponential kernel, the latter becomes $\widetilde{\mathcal O}(\epsilon^{1/4})$. Our analysis also isolates how the bound-suggested constant exploration parameter depends on the drift rate. In the regret bounds of \citet{bogunovic2016time}, the prescribed exploration schedule is independent of $\epsilon$. In contrast, balancing the upper-bound scaling suggests that the bound-suggested constant $\beta$ increases as $\epsilon$ decreases. This does not by itself imply uniformly more exploration, since the UCB rule depends on the product of $\sqrt{\beta}$ and the posterior standard deviation. Rather, when drift is slower, past observations remain informative for longer and the posterior variance is typically smaller. A larger $\beta$ compensates for this reduced variance by assigning more value to reducing the remaining uncertainty, which is also more consequential because information acquired now remains useful over a longer effective horizon.

The dependence of the bound-suggested exploration parameter on the temporal dynamics also gives a simple rule for changes in sampling frequency. If observations are collected every $\Delta$ time units, larger $\Delta$ corresponds to greater effective drift between successive evaluations. Temporal correlations then decay faster, past observations become obsolete more quickly, and the analysis predicts a smaller bound-suggested exploration parameter.

The remainder of the paper is organized as follows. Section~\ref{sec:problem} introduces the time-varying GP model, the domain regularity assumptions, and the temporally discounted GP-UCB policy. Section~\ref{sec:improved_regret} presents the expected-regret bound and the post-hoc and prescribed-confidence guarantees for realized regret. Section~\ref{sec:implications} derives the drift-dependent information-gain and regret rates and analyzes the bound-suggested exploration parameter. Section~\ref{sec:simulation_study} provides a supporting empirical simulation study. Section~\ref{sec:conclusion} provides the conclusion. Additional proofs and technical results are deferred to the appendices.

\section{Problem setting}
\label{sec:problem}

We start by introducing the dynamic regret setting. Let $\mathcal{S}\subseteq\mathbb{R}^d$ be compact. At each round $t=1,2,\ldots$, the algorithm selects a query point $x_t\in\mathcal{S}$ and observes
\begin{equation}
y_t = f_t(x_t) + z_t, 
\qquad 
z_t \sim \mathcal{N}(0,\sigma^2)\ \text{i.i.d.},
\label{eq:obs_model}
\end{equation}
where $f_t:\mathcal{S}\to\mathbb{R}$ is an unknown reward function that changes with time. Let
\[
x_t^\star \in \arg\max_{x\in\mathcal{S}} f_t(x)
\]
denote a maximizer at round $t$. The instantaneous regret is
\begin{equation}
r_t := f_t(x_t^\star)-f_t(x_t) \geq 0,
\label{eq:inst_regret}
\end{equation}
and the cumulative regret over $T$ rounds is
\begin{equation}
R_T := \sum_{t=1}^T r_t.
\label{eq:cum_regret}
\end{equation}
In the non-stationary setting, sublinear cumulative regret is not expected in general, so the main objective is to control the linear growth rate and its dependence on the temporal properties of $f_t$. We model the function sequence $f_t$ using the time-varying Gaussian process construction of \citet{bogunovic2016time}. Let $(g_t)_{t\ge 1}$ be i.i.d.\ draws from $\mathcal{GP}(0,k_S)$, where $k_S$ is a spatial kernel on $\mathcal{S}$. The latent process is defined by
\begin{equation}
f_1(x)=g_1(x),\qquad
f_{t+1}(x)=\sqrt{1-\epsilon}\,f_t(x)+\sqrt{\epsilon}\,g_{t+1}(x),
\quad t\ge 1,
\label{eq:tvgp}
\end{equation}
with drift parameter $\epsilon\in(0,1]$. For each fixed $t$, the marginal law remains $\mathcal{GP}(0,k_S)$, while the temporal dependence decays geometrically. Equivalently, $\{f_t(x)\}_{(x,t)}$ is a Gaussian process on $\mathcal{S}\times\mathbb{N}$ with separable kernel
\begin{equation}
k\big((x,t),(x',t')\big)
=
k_S(x,x')\,k_T(t,t'),
\qquad
k_T(t,t')
=
(\sqrt{1-\epsilon})^{|t-t'|}.
\label{eq:sep_kernel}
\end{equation}
We refer to \eqref{eq:sep_kernel} as the time-decayed kernel, since its temporal factor discounts correlations geometrically with temporal separation.
We assume throughout that the prior variance is uniformly bounded,
\[
k_S(x,x)\le 1
\qquad
\text{for all } x\in\mathcal{S}.
\]

Given past data
\[
\mathcal{H}_{t-1}=\{(x_i,y_i)\}_{i=1}^{t-1},
\qquad
y_{1:t-1}=(y_1,\ldots,y_{t-1})^\top,
\]
Gaussian process regression under the kernel \eqref{eq:sep_kernel} yields a Gaussian posterior for $f_t(x)$ with mean and variance
\begin{align}
\mu_{t-1}(x)
&=
\tilde k_{t-1}(x)^\top
\big(\tilde K_{t-1}+\sigma^2 I\big)^{-1}
y_{1:t-1},
\label{eq:post_mean}\\
\sigma_{t-1}^2(x)
&=
k_S(x,x)
-
\tilde k_{t-1}(x)^\top
\big(\tilde K_{t-1}+\sigma^2 I\big)^{-1}
\tilde k_{t-1}(x),
\label{eq:post_var}
\end{align}
where
\begin{equation}
\tilde k_{t-1}(x)
:=
\big[
k_S(x_i,x)\,(\sqrt{1-\epsilon})^{|t-i|}
\big]_{i=1}^{t-1},
\qquad
\tilde K_{t-1}
:=
\big[
k_S(x_i,x_j)\,(\sqrt{1-\epsilon})^{|i-j|}
\big]_{i,j=1}^{t-1}.
\label{eq:discounted_kernels}
\end{equation}
These expressions are obtained by evaluating the spatio-temporal kernel \eqref{eq:sep_kernel} at the observed pairs $(x_i,i)$ and the prediction point $(x,t)$. The factors $(\sqrt{1-\epsilon})^{|t-i|}$ account for the fact that older observations are less informative about the current function $f_t$.

As in standard GP bandit analysis, we measure cumulative uncertainty by the time-varying maximum information gain \citep{bogunovic2016time}, which controls sums of posterior variances in regret bounds \citep[Lemma~5.4]{srinivas2012information}. For a design set
\[
A=\{(x_1,1),\ldots,(x_T,T)\},
\]
let $f_A=(f_1(x_1),\ldots,f_T(x_T))^\top$ and $y_A=(y_1,\ldots,y_T)^\top$. Since $y_A=f_A+z_A$ with independent Gaussian noise,
\begin{equation}
I(y_A;f_A)
=
\frac{1}{2}\log\det\big(I+\sigma^{-2}\tilde K_A\big),
\label{eq:mutual_info}
\end{equation}
where $\tilde K_A$ is the covariance matrix of $f_A$ induced by \eqref{eq:sep_kernel}. We define the time-varying maximum information gain as
\begin{equation}
\tilde\gamma_T
:=
\max_{x_{1:T}\in\mathcal{S}^T} I(y_A;f_A),
\label{eq:max_info_gain}
\end{equation}

To control the discretization error incurred on a continuous domain, we impose the following standard GP-UCB regularity condition \citep{srinivas2012information,bogunovic2016time}.

\begin{assumption}[Derivative tail bound]
\label{ass:smoothness}
There exist constants $a\ge 1$ and $b>0$ such that, for each $t\ge 1$, each spatial dimension $j\in\{1,\ldots,d\}$, and all $L>0$,
\begin{equation}
\mathbb{P}\Big\{
\sup_{x\in\mathcal{S}}
\big|
\frac{\partial f_t(x)}{\partial x^{(j)}}
\big|
>L
\Big\}
\le
a\exp\Big(-\big(\frac{L}{b}\big)^2\Big).
\label{eq:smoothness}
\end{equation}
\end{assumption}

This condition holds for many commonly used kernels, including the squared exponential kernel and Mat\'ern kernels with smoothness parameter $\nu>2$, under appropriate differentiability conditions on $k_S$; see \citet[Theorem~5]{ghosal2006posterior}. The TV-GP-UCB selection rule is given in Algorithm~\ref{alg:TV}. At round $t$, the algorithm selects
\begin{equation}
x_t \in \arg\max_{x\in\mathcal{S}}
\Big(
\mu_{t-1}(x)+\sqrt{\beta_t}\,\sigma_{t-1}(x)
\Big).
\label{eq:gpucb}
\end{equation}
The main result of this paper shows that under the time-varying model \eqref{eq:tvgp}, one can take $\beta_t\equiv\beta$ constant and obtain a regret guarantee with improved long-horizon scaling. This contrasts with earlier analyses that rely on time-increasing exploration schedules.

\begin{algorithm}[ht]
\caption{Constant-$\beta$ Time-Varying GP-UCB (TV-GP-UCB)}
\label{alg:TV}
\begin{algorithmic}[1]
\Require Domain $\mathcal{S}$, GP prior ($0$, $k_S$, $\sigma^2$), drift parameter $\epsilon \in (0,1]$, constant exploration parameter $\beta > 0$
\For {$t = 1,2, \dots$}
\State Choose $x_t \in \arg\max_{x\in\mathcal{S}} \big( \mu_{t-1}(x) + \sqrt{\beta} \sigma_{t-1}(x) \big)$
\State Observe $y_t = f_t(x_t) + z_t$
\State Update posterior mean $\mu_t$ and variance $\sigma_t^2$ via \eqref{eq:post_mean} and \eqref{eq:post_var}
\EndFor
\end{algorithmic}
\end{algorithm}
\section{Improved Regret Bound}
\label{sec:improved_regret}

Static GP-UCB analyses usually prove a single confidence event that holds simultaneously over all rounds and all points in a discretization. To keep this event at fixed probability over a horizon $T$, the per-round failure probabilities must be summable, which forces the exploration parameter $\beta_t$ to grow with $t$. Existing analyses of the time-varying GP model \eqref{eq:tvgp}, in particular \citet{bogunovic2016time}, use the same time-increasing exploration schedule. This introduces horizon-dependent tuning, leads to average-regret bounds that grow with $T$, and obscures the dependence of exploration on the drift rate $\epsilon$.

We avoid this global-in-time confidence event. When $\epsilon>0$, older observations are progressively discounted, so posterior uncertainty need not decrease monotonically. This built-in forgetting prevents the algorithm from becoming permanently overconfident at a fixed point and removes the need for a time-increasing exploration schedule in the analysis. We prove that Algorithm~\ref{alg:TV} with a constant exploration parameter attains an $\mathcal{O}(T)$ expected cumulative regret bound for every fixed $\epsilon>0$. The proof combines a per-round UCB regret bound on a local confidence event with an algorithm-agnostic bound on the rare complement event.

\subsection{Expected regret with a constant exploration parameter}

\begin{theorem}[Expected regret for constant-$\beta$ TV-GP-UCB]
\label{thrm:improved_regret}
Let $\mathcal{S} \subseteq [0,r]^d$ be compact and convex. Assume the time-varying GP model \eqref{eq:tvgp} with drift $\epsilon\in(0,1]$, spatial kernel $k_S$ satisfying $k_S(x,x)\le 1$ for all $x\in\mathcal{S}$, and the tail condition \eqref{eq:smoothness}. Assume observations are generated by
\[
y_t=f_t(x_t)+z_t,
\qquad
z_t\sim\mathcal{N}(0,\sigma^2)
\]
with independent noise. For $t\ge 1$, define the spatial span
\[
W_t := \sup_{x\in\mathcal{S}} f_t(x)-\inf_{x\in\mathcal{S}} f_t(x),
\]
let $\mu_w=\mathbb{E}[W_t]$, and let $\sigma_w^2$ be any finite sub-Gaussian variance proxy for $W_t$, so that
\[
\mathbb{P}\{W_t-\mu_w>u\}\le \exp\!\left(-\frac{u^2}{2\sigma_w^2}\right)
\qquad \text{for all } u>0.
\]
Fix $\delta\in(0,1)$ and $\tau\in\mathbb{N}^+$. If the evaluation points $\{x_t\}_{t=1}^T$ are selected according to Algorithm~\ref{alg:TV} with any constant exploration parameter satisfying
\begin{equation}
\beta \ge 2\Big(\log\!\big(\tfrac{6}{\delta}\big)+ d\log(\tau)\Big).
\label{eq:beta_threshold}
\end{equation}
then the expected cumulative regret satisfies
\begin{equation}
\mathbb{E}[R_T] \le \sqrt{C_1 T \beta\,\tilde{\gamma}_T}+ T C_2 + T\delta C_3,
\label{eq:R_T_expect_thrm_improved}
\end{equation}
where $\tilde{\gamma}_T=\tilde{\gamma}_T(\mathcal{S},k_S,\epsilon,\sigma)$ is the size-$T$ maximum information gain \eqref{eq:max_info_gain} under the time-decayed kernel, and
\[
C_1=\frac{8}{\log(1+\sigma^{-2})},\qquad
C_2=\frac{rdb}{\tau} \sqrt{\log\Big(\frac{3ad}{\delta}\Big)},\qquad
C_3=\mu_w+2\sigma_w\sqrt{2\log(e/\delta)}.
\]
\end{theorem}
For comparison, since $r_t\le W_t$ for every policy, the policy-independent span bound is
\[
R_T\le\sum_{t=1}^T W_t,
\qquad
\frac{\mathbb E[R_T]}{T}\le\mu_w.
\]
Because each $f_t$ has marginal law $\mathcal{GP}(0,k_S)$, the quantity $\mu_w$ is independent of $\epsilon$. Thus Theorem~\ref{thrm:improved_regret} improves this policy-independent baseline whenever
\[
\sqrt{C_1\beta\,\frac{\tilde\gamma_T}{T}}+C_2+\delta C_3<\mu_w.
\]
Under the drift-dependent optimization in Section~\ref{subsec:info_gain_rates}, the algorithm-specific coefficient decreases as $\epsilon$ decreases in the persistent-drift regime, whereas $\mu_w$ does not.

The span constants $\mu_w$ and $\sigma_w$ are finite by Borell--TIS, since $W_t$ is the supremum of the centered Gaussian difference process $f_t(x)-f_t(y)$ over $\mathcal{S}\times\mathcal{S}$ and the sample paths are almost surely bounded on the compact domain.

The proof, given in Appendix~\ref{app:Improved_regret}, constructs a local confidence event $\mathcal{E}_t$ at each round rather than a single event over the full horizon. The event combines posterior concentration at the selected point, posterior concentration over a spatial grid of size at most $\tau^d$, and the derivative tail bound \eqref{eq:smoothness}. On $\mathcal{E}_t$, the instantaneous regret satisfies
\[
r_t \le 2\sqrt{\beta}\,\sigma_{t-1}(x_t)+C_2,
\]
where $C_2$ is the discretization error. On the complement $\mathcal{E}_t^c$, we use only the deterministic inequality $r_t\le W_t$. The sub-Gaussian tail of $W_t$ gives
\[
\mathbb{E}\big[W_t\mathbbm{1}_{\mathcal{E}_t^c}\big]\le \delta C_3.
\]
Therefore,
\[
\mathbb{E}[r_t]
\le
2\sqrt{\beta}\,\mathbb{E}[\sigma_{t-1}(x_t)]
+
C_2
+
\delta C_3.
\]
Summing over $t$ and applying the standard information-gain bound on posterior variances yields \eqref{eq:R_T_expect_thrm_improved}.

We use the notation $\delta$ and $\tau$ from \citet{bogunovic2016time}: $\delta$ controls the confidence level and $\tau^d$ is the spatial grid size. The distinction is that in their analysis $\delta$ is tied to a single confidence event over the full horizon, and therefore determines the time-varying exploration schedule $\beta_t$. Here $\delta$ only controls the local event $\mathcal{E}_t$ used in the expected-regret decomposition. Thus $\delta$ and $\tau$ are analysis parameters: each choice gives a corresponding constant $\beta$ and expected-regret bound.

Since $\tilde{\gamma}_T$ depends on the drift rate $\epsilon$, the bound-balancing choice of $(\delta,\tau)$ generally depends on $\epsilon$ as well. This is how the bound-suggested constant exploration parameter inherits a dependence on the temporal dynamics. Section~\ref{subsec:Opt_beta_eps} makes this explicit: under the scaling $\tilde{\gamma}_T=\widetilde{\mathcal{O}}(T\epsilon^\alpha)$, the resulting bound-suggested choice satisfies $\beta_{\mathrm{opt}}=\mathcal{O}(\log(1/\epsilon))$.

\subsection{Realized cumulative regret and confidence guarantees}
\label{subsec:prob_regret_bound}

Theorem~\ref{thrm:improved_regret} bounds the expected cumulative regret. For realized regret, we distinguish two uses of a constant exploration parameter. In the post-hoc setting, $\beta$ is fixed without reference to the reporting confidence level $\omega$, whereas in the prescribed-confidence setting, $\omega$ is specified before the run and is used to select a constant $\beta$. In both cases, $\beta$ is independent of $t$ and $T$.

Standard analyses of time-varying GP bandits construct probabilistic guarantees using a time-increasing exploration schedule, typically $\beta_t=\mathcal{O}(\log(t/\omega))$. This enforces a global union bound over the horizon but introduces an $\mathcal{O}(\sqrt{\log T})$ factor in the leading information-gain term. We instead retain a constant exploration parameter and control the accumulated contribution of the local confidence failures.

\paragraph{Post-hoc confidence.}
Let
\[
U_T := \sqrt{C_1 T \beta \tilde{\gamma}_T} + T C_2
\]
denote the deterministic contribution from the UCB confidence widths and the discretization error. The proof bounds the realized pathwise regret by combining this term with the accumulated span over failed confidence events and applying Markov's inequality to the latter.

\begin{proposition}[Realized cumulative regret]
\label{prop:realized_regret}
Assume the setting and parameter choices of Theorem~\ref{thrm:improved_regret}. For any target failure probability $\omega \in (0,1)$, the realized cumulative regret $R_T$ satisfies the following bound with probability at least $1-\omega$:
\begin{equation}
R_T \;\le\; U_T \;+\; \frac{T \delta C_3}{\omega}.
\label{eq:bound_markov}
\end{equation}
\end{proposition}

The proof is deferred to Appendix~\ref{app:proof_realized_regret}. Here $\delta$ is the per-round analysis parameter that determines the constant $\beta$, whereas $\omega$ enters only in the reported bound. Thus the algorithm can be run with $\beta$ fixed independently of $\omega$, and the bound can subsequently be evaluated at any fixed reporting level $\omega$ chosen independently of the realized data. The cost of this post-hoc guarantee is the polynomial $1/\omega$ dependence.

\paragraph{Prescribed confidence.}
If the target failure probability is specified before the run, the same result gives logarithmic dependence on $1/\omega$ while retaining a constant exploration parameter.

\begin{corollary}[Prescribed-confidence realized regret]
\label{cor:prescribed_confidence}
Assume the setting of Theorem~\ref{thrm:improved_regret} and suppose that, in the regime under consideration,
\[
\frac{\tilde\gamma_T}{T}
=
\widetilde{\mathcal O}(\epsilon^\alpha),
\qquad \alpha>0.
\]
Fix $\omega\in(0,1)$ and choose
\[
\delta=\omega\epsilon^{\alpha/2},
\qquad
\tau=\left\lceil\epsilon^{-\alpha/2}\right\rceil,
\]
together with the constant exploration parameter
\[
\beta
=
2\left(
\log\!\frac{6}{\omega\epsilon^{\alpha/2}}
+
d\log\left\lceil\epsilon^{-\alpha/2}\right\rceil
\right)
=
\mathcal O\!\left(
\log(1/\omega)+\log(1/\epsilon)
\right).
\]
Then, with probability at least $1-\omega$,
\begin{equation}
\frac{R_T}{T}
=
\widetilde{\mathcal O}\!\left(
\epsilon^{\alpha/2}
\sqrt{
1+\log(1/\omega)+\log(1/\epsilon)
}
\right).
\label{eq:prescribed_confidence_rate}
\end{equation}
In particular, $\beta$ remains independent of $t$ and $T$.
\end{corollary}

Proposition~\ref{prop:realized_regret} gives
\[
\frac{R_T}{T}
\le
\sqrt{C_1\beta\,\frac{\tilde\gamma_T}{T}}
+
C_2
+
\frac{\delta C_3}{\omega}.
\]
Under the choices above, the three terms are respectively
\[
\widetilde{\mathcal O}\!\left(
\epsilon^{\alpha/2}
\sqrt{\log(1/\omega)+\log(1/\epsilon)}
\right),
\qquad
\mathcal O\!\left(
\epsilon^{\alpha/2}
\sqrt{1+\log(1/\omega)+\log(1/\epsilon)}
\right),
\]
and
\[
\mathcal O\!\left(
\epsilon^{\alpha/2}
\sqrt{1+\log(1/\omega)+\log(1/\epsilon)}
\right),
\]
which yields \eqref{eq:prescribed_confidence_rate}.

\paragraph{Policy-independent fallback.}
Independently of the sampling policy,
\[
R_T\le\sum_{t=1}^T W_t.
\]
As established in Appendix~\ref{app:span_concentration}, with $\rho=\sqrt{1-\epsilon}$, this gives with probability at least $1-\omega$,
\begin{equation}
R_T
\le
\sum_{t=1}^T W_t
\le
T\mu_w
+
\sqrt{
8T\frac{1+\rho}{1-\rho}\log(1/\omega)
}.
\label{eq:span_fallback_main}
\end{equation}
This bound does not use the sampling decisions and is therefore not a replacement for the algorithm-specific guarantees above.

\section{Drift-dependent regret rates and exploration}
\label{sec:implications}

Theorem~\ref{thrm:improved_regret} gives a constant-exploration regret bound whose leading term is controlled by $\tilde\gamma_T/T$. We first combine it with the time-varying information-gain bound in Appendix~\ref{app:max_info_gain} to derive explicit drift-dependent rates, and then examine the corresponding bound-suggested choice of $\beta$.

\subsection{Information-gain and regret rates}
\label{subsec:info_gain_rates}

The time-varying information-gain bound in Appendix~\ref{app:max_info_gain} gives an explicit dependence on the drift rate. Suppose the static maximum information gain satisfies
\[
\gamma_N=\widetilde{\mathcal O}(N^q),
\qquad q\in[0,1).
\]
At the polynomial scale, the corresponding block length is
\[
N_\epsilon=\epsilon^{-\frac{2}{4-q}}.
\]
In the persistent-drift regime where $T$ exceeds this effective block length, Corollary~\ref{cor:info_gain_scaling} gives
\begin{equation}
\frac{\tilde\gamma_T}{T}
=
\widetilde{\mathcal O}\!\left(
\epsilon^{\frac{2(1-q)}{4-q}}
\right).
\label{eq:info_gain_main_rate}
\end{equation}
Balancing the remaining terms in Theorem~\ref{thrm:improved_regret} as in Section~\ref{subsec:Opt_beta_eps} then yields
\begin{equation}
\frac{\mathbb E[R_T]}{T}
=
\widetilde{\mathcal O}\!\left(
\epsilon^{\frac{1-q}{4-q}}
\right).
\label{eq:regret_main_rate}
\end{equation}

For the squared exponential kernel, $\gamma_N$ is polylogarithmic in $N$ \citep{srinivas2012information,vakili2021information,iwazaki2025improved}, so $q=0$ up to logarithmic factors and
\[
\frac{\tilde\gamma_T}{T}
=
\widetilde{\mathcal O}(\epsilon^{1/2}),
\qquad
\frac{\mathbb E[R_T]}{T}
=
\widetilde{\mathcal O}(\epsilon^{1/4}).
\]
For a Mat\'ern kernel with smoothness $\nu>2$, the static rate
$q=d/(2\nu+d)$ \citep{vakili2021information,iwazaki2025improved} gives
\[
\frac{\tilde\gamma_T}{T}
=
\widetilde{\mathcal O}\!\left(
\epsilon^{\frac{4\nu}{8\nu+3d}}
\right),
\qquad
\frac{\mathbb E[R_T]}{T}
=
\widetilde{\mathcal O}\!\left(
\epsilon^{\frac{2\nu}{8\nu+3d}}
\right).
\]

These rates are not fixed-$T$, $\epsilon\to0$ statements. If $N_\epsilon>T$, taking $\tilde N=T$ in Theorem~\ref{thm:TV_inf_gain_fixed} gives
\[
\tilde\gamma_T
\le
\gamma_T
+
\frac{\sigma^{-2}\epsilon\,T\sqrt{T^2-1}}{2\sqrt{3}}
\sqrt{\gamma_T},
\]
and hence $\tilde\gamma_T\to\gamma_T$ for fixed $T$ as $\epsilon\to0$. The constant-exploration result is therefore a positive-drift result and is not claimed to recover the usual static sublinear-regret guarantee at $\epsilon=0$; the static case requires a separate analysis \citep{takeno2025regret}.

\paragraph{Comparison with previous rates.}
The information-gain improvement is separate from the confidence-schedule improvement. If $\gamma_N=\widetilde{\mathcal O}(N^q)$, the original block perturbation in \citet{bogunovic2016time} gives the first rate below. The $\mathcal O(\epsilon N^{5/2})$ block order from the unit-time specialization of \citet{imamura2020time} gives the second, and Theorem~\ref{thm:TV_inf_gain_fixed} gives the third:
\[
\frac{\tilde\gamma_T}{T}
=
\widetilde{\mathcal O}\!\left(
\epsilon^{\frac{1-q}{3-q}}
\right),
\qquad
\widetilde{\mathcal O}\!\left(
\epsilon^{\frac{2(1-q)}{5-2q}}
\right),
\qquad
\widetilde{\mathcal O}\!\left(
\epsilon^{\frac{2(1-q)}{4-q}}
\right).
\]
The time-growing confidence schedule contributes a factor of order
$\sqrt{\log(T/\omega)}$ to the leading regret term, whereas the prescribed-confidence constant-$\beta$ choice uses
\[
\beta
=
\mathcal O\!\left(
\log(1/\omega)+\log(1/\epsilon)
\right),
\]
so the corresponding confidence factor is independent of $T$.

\subsection{Bound-suggested $\beta$ dependence on $\epsilon$}
\label{subsec:Opt_beta_eps}

Here $\beta$ denotes any constant exploration parameter satisfying \eqref{eq:beta_threshold}, while $\beta_{\mathrm{opt}}$ denotes the bound-suggested value obtained from the scaling choice of the analytical parameters $(\delta,\tau)$, rather than a policy-level optimal parameter. The drift parameter $\epsilon$ is already required to construct the TV-GP posterior, so using it to guide $\beta$ introduces no additional model parameter. The scaling $\beta_{\mathrm{opt}}=\mathcal{O}(\log(1/\epsilon))$ below is a bound-balancing recommendation rather than a validity condition; Theorem~\ref{thrm:improved_regret} holds for any constant $\beta$ satisfying \eqref{eq:beta_threshold}. The analysis of \citet{bogunovic2016time} specifies the time-growing schedule
\begin{equation}
\label{eq:thm_tv_beta}
\beta_t = 2 \log \frac{\pi^2 t^2}{2 \delta} + 2 d \log \!\left(r d b t^2 \sqrt{\log \frac{d a \pi^2 t^2}{2 \delta}}\right),
\end{equation}
where $(a,b,r,d)$ are spatial domain and smoothness constants, and $\delta$ is a failure probability. This schedule is determined by the spatial discretization argument and is not optimized as a function of the drift parameter $\epsilon$.

Our theorem does not prescribe $\beta$ as an explicit function of $\epsilon$ directly. Instead, $\epsilon$ enters through the time-decayed maximum information gain $\tilde\gamma_T$. Balancing the upper-bound terms through the analytical parameters $(\delta,\tau)$ therefore induces an $\epsilon$-dependent bound-suggested constant $\beta_{\mathrm{opt}}(\epsilon)$. To make this dependence explicit, suppose that, for the chosen spatial kernel and in the long-horizon regime where \(T\) exceeds the effective block length, the time-decayed information gain satisfies
\begin{equation}
\tilde\gamma_T = \widetilde{\mathcal{O}}\!\left(T \epsilon^{\alpha}\right),
\qquad \alpha>0.
\label{eq:gamma_scaling_eps}
\end{equation}
For fixed \(T\), this scaling should not be interpreted as an \(\epsilon\to0\) limit, since \(\tilde\gamma_T\) then approaches the static information gain. Appendix~\ref{app:max_info_gain} verifies such a scaling for common kernels through the standard growth rate of the time-invariant information gain. Taking equality in \eqref{eq:beta_threshold}, $\beta = 2(\log(6/\delta)+d\log\tau)$, while $C_2=(rdb/\tau)\sqrt{\log(3ad/\delta)}$. Thus the three per-round contributions in the expected regret bound are controlled by $\sqrt{\beta\,\tilde\gamma_T/T}$, $C_2$, and $\delta C_3$.

Under \eqref{eq:gamma_scaling_eps}, the information-gain contribution per round scales as $\sqrt{\beta}\,\epsilon^{\alpha/2}$ up to logarithmic factors. For $\epsilon\in(0,1)$, the choice below balances the polynomial dependence on $\epsilon$ across the information-gain, discretization, and failure-event terms and is sufficient to identify the logarithmic dependence of $\beta_{\mathrm{opt}}$ on $1/\epsilon$:
\begin{equation}
\delta = \epsilon^{\alpha/2},
\qquad
\tau = \left\lceil \epsilon^{-\alpha/2}\right\rceil.
\label{eq:delta_tau_choice}
\end{equation}

With this choice, the discretization and failure-event contributions match the same polynomial order, up to logarithmic factors:
\begin{align}
C_2
&=
\mathcal{O}\!\left(\epsilon^{\alpha/2}\sqrt{\log(\epsilon^{-1})}\right),
\label{eq:C2_scaling_eps}\\
\delta C_3
&=
\mathcal{O}\!\left(\epsilon^{\alpha/2}\sqrt{\log(\epsilon^{-1})}\right).
\label{eq:deltaC3_scaling_eps}
\end{align}

Substituting \eqref{eq:delta_tau_choice} into \eqref{eq:beta_threshold} at equality gives
\begin{equation}
\beta_{\mathrm{opt}}(\epsilon)
\approx
2\log(6)+\alpha(1+d)\log\!\Big(\frac{1}{\epsilon}\Big).
\label{eq:beta_log_eps_scaling}
\end{equation}
Across environments with the same spatial kernel but different drift rates, the bound suggests that $\beta_{\mathrm{opt}}$ grows logarithmically with $1/\epsilon$. This scaling should be interpreted together with the posterior variance: smaller $\epsilon$ also makes past data more informative and typically reduces $\sigma_{t-1}(x)$.

\subsection{Sampling frequency and effective drift}
\label{subsec:sampling_frequency}

The same calculation gives a simple rule for changing the sampling frequency. Suppose the underlying process evolves at drift rate $\epsilon$, but observations are collected every $\Delta\in\mathbb{N}^+$ time units. Then the correlation between successive observations is $(\sqrt{1-\epsilon})^\Delta$, which corresponds to the effective discrete-time drift
\begin{equation}
\epsilon_\Delta = 1-(1-\epsilon)^\Delta.
\label{eq:eps_delta}
\end{equation}
Larger $\Delta$ increases $\epsilon_\Delta$, so past observations become less relevant. Substituting this effective drift into \eqref{eq:beta_log_eps_scaling} yields the heuristic
\begin{equation}
\beta_{\mathrm{opt}}(\epsilon_\Delta)
\propto
\log\!\Big(\frac{1}{\epsilon_\Delta}\Big)
=
\log\!\Big(\frac{1}{1-(1-\epsilon)^\Delta}\Big).
\label{eq:beta_delta_heuristic}
\end{equation}
This adjustment reflects the fact that less frequent sampling increases the effective drift between observations and reduces the useful lifetime of past data.

\subsection{Average-reward perspective}
\label{sec:reward_maximizing_perspective}

The dynamic setting also changes the interpretation of the benchmark. Let $M_t := \sup_{x\in\mathcal S} f_t(x)$ denote the round-$t$ optimal reward. In the static case, $M_t$ is the same random variable for all $t$, so the benchmark remains tied to a single realization of the objective. Regret bounds must therefore account for the joint distribution of the realized function and the algorithm's decisions.

When $\epsilon>0$, this coupling is weaker in the long run. The sequence $(M_t)_{t\ge 1}$ is stationary and has exponentially decaying temporal dependence. As shown in Appendix~\ref{app:oracle_concentration},
\[
\frac{1}{T}\sum_{t=1}^T M_t
\xrightarrow[T\to\infty]{\Pr}
\mathbb{E}[M_1].
\]
Thus the oracle part of the regret concentrates around a deterministic quantity. Consequently, asymptotic average regret can be viewed as the gap between the deterministic benchmark $\mathbb{E}[M_1]$ and the algorithm's realized average reward. Cumulative regret remains the relevant finite-horizon criterion, but this perspective separates the intrinsic variability of the moving optimum from the performance loss due to the algorithm's decisions.

\section{Simulation study}
\label{sec:simulation_study}

We empirically evaluate constant-$\beta$ GP-UCB under the TV-GP model on the spatial domain $\mathcal{S}=[0,1]^2$. The domain is discretized using a $20\times 20$ grid, giving a state dimension of $N=400$. The latent function uses a radial basis function (RBF, squared exponential) kernel with lengthscale $l=0.1$ and prior variance $1.0$, while the observation noise variance is fixed at $\sigma^2=0.01$. We simulate $T=10{,}000$ rounds for each parameter pair $(\epsilon,\beta)$ and record the time-averaged regret $R_T/T$ over 100 independent trials. This horizon is used as a long-run stress test rather than as an application-specific benchmark; it is made feasible by the Kalman implementation described below. The simulation setup and compute resources are reported in Appendix~\ref{app:simulation_reproducibility}, and the accompanying code archive contains the scripts used to reproduce Figure~\ref{fig:simulation-study}.

\paragraph{Exact inference via Kalman filtering.}
A direct GP implementation would require repeatedly conditioning on the full observation history, making long-horizon simulations computationally expensive. To avoid this bottleneck, we exploit the autoregressive structure of the TV-GP prior. On the finite grid, the latent function can be represented as a state vector $f_t\in\mathbb{R}^{400}$, and the TV-GP model becomes a linear Gaussian state-space model. We therefore compute the GP posterior exactly using a Kalman filter. This reduces the per-round posterior update from the $\mathcal{O}(t^3)$ history-dependent cost of naive GP regression to a fixed-grid $\mathcal{O}(N^2)$ recursion. The full recursion is given in Appendix~\ref{app:kalman_filter_simulations}.

\paragraph{Results.}
The left panel of Figure~\ref{fig:simulation-study} displays the average regret $R_T/T$ as a function of the constant exploration parameter $\beta$ for three drift rates, $\epsilon\in\{10^{-1},10^{-2},10^{-3}\}$. For each drift rate, let $\hat{\beta}_{\mathrm{opt}}$ denote the value of $\beta$ on the tested grid that minimizes the empirical average regret. The main empirical test is whether $\hat{\beta}_{\mathrm{opt}}$ increases with $-\log(\epsilon)$ as predicted by Section~\ref{subsec:Opt_beta_eps}. The empirical minimizer shifts toward larger values as $\epsilon$ decreases: $\hat{\beta}_{\mathrm{opt}}\approx 2.12$ for $\epsilon=0.1$, $\hat{\beta}_{\mathrm{opt}}\approx 4.34$ for $\epsilon=0.01$, and $\hat{\beta}_{\mathrm{opt}}\approx 6.36$ for $\epsilon=0.001$.

Because the exact theoretical schedule in Equation~\eqref{eq:thm_tv_beta} is overly conservative in practice \citep{srinivas2012information, bogunovic2016time}, empirical evaluations often use tuned logarithmic schedules. Following \citet{bogunovic2016time}, we compare against $\beta_t=c_1\log(c_2t)$ with $c_1=0.8$ and $c_2=4$. Across all three drift rates, a broad interval of constant $\beta$ values achieves lower average regret than this time-growing heuristic over the $10{,}000$-step horizon.

\begin{figure}[t]
    \centering
    \includegraphics[width=\linewidth]{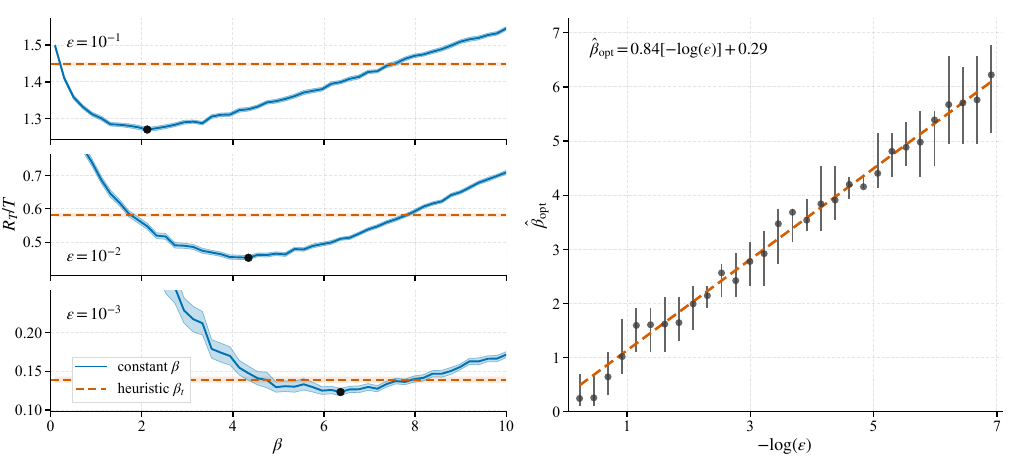}
    \caption{
    Simulation study for constant-$\beta$ GP-UCB under the time-varying GP model.
    Left: time-averaged regret $R_T/T$ over $T=10{,}000$ rounds as a function of the constant exploration parameter $\beta$ for three drift rates $\epsilon\in\{10^{-1},10^{-2},10^{-3}\}$; shaded bands denote 95\% confidence intervals over 100 independent trials, and dashed horizontal lines show the logarithmic time-varying-$\beta_t$ heuristic. 
    Right: bootstrapped estimates of the empirical regret-minimizing constant $\hat{\beta}_{\mathrm{opt}}$ versus $-\log(\epsilon)$ across the full $\epsilon$ grid. Points denote mean bootstrap estimates, vertical bars denote 95\% percentile intervals, and the dashed line shows a least-squares linear fit used to visualize the log-linear trend.
    }
    \label{fig:simulation-study}
\end{figure}

The right panel of Figure~\ref{fig:simulation-study} examines this dependence across the full grid of drift rates. The relationship between $\hat{\beta}_{\mathrm{opt}}$ and $-\log(\epsilon)$ is approximately linear, matching the qualitative prediction $\beta_{\mathrm{opt}}(\epsilon)\propto \log(1/\epsilon)$ from Equation~\eqref{eq:beta_log_eps_scaling}. This agreement is nontrivial: the theory balances terms in an upper bound, and such bounds need not match the exploration scale that performs best in finite simulations.

We do not interpret the fitted slope as an estimate of the information-gain exponent $\alpha$. Such an interpretation would require the analytical discretization choice $\tau=\tau(\epsilon)$ and the global spatial union-bound factor in Equation~\eqref{eq:beta_log_eps_scaling} to be tight. Neither condition holds here: the simulation uses a fixed $20\times20$ action grid, and the adaptive policy concentrates its evaluations near high-reward regions rather than exploring the full domain uniformly. Thus the slope is best viewed as empirical support for the logarithmic dependence of the regret-minimizing constant on the drift rate, not as a quantitative estimate of the information-gain scaling.

\section{Conclusion}
\label{sec:conclusion}

Existing theoretical guarantees for Gaussian process bandits in time-varying environments typically require exploration schedules that grow with time. We showed that, under the TV-GP drift model, a constant exploration parameter can instead be analyzed using local confidence events at each round rather than a single horizon-wide confidence event. The resulting expected-regret bound has a drift-dependent leading term controlled by the time-varying information gain, and the same local-confidence analysis yields realized-regret guarantees with $\beta$ constant in $t$ and $T$.

We also derived a sharper bound on the time-varying maximum information gain. If $\gamma_N=\widetilde{\mathcal O}(N^q)$, the resulting persistent-drift average-regret rate is $\widetilde{\mathcal O}(\epsilon^{(1-q)/(4-q)})$; for the squared exponential kernel this gives $\widetilde{\mathcal O}(\epsilon^{1/4})$. These rates apply when the optimized block length fits within the horizon and are not fixed-$T$, $\epsilon\to0$ statements. Separately, balancing the remaining terms in the regret bound suggests a bound-suggested constant exploration parameter with $\beta_{\mathrm{opt}}(\epsilon)\propto\log(1/\epsilon)$. This provides a simple tuning principle when changing the sampling frequency changes the effective drift between observations. The same Markov structure also implies exponential covariance decay for the oracle-value process, so long-run average performance can be compared against a deterministic benchmark.

Several extensions remain open. Most importantly, the present analysis relies on the geometric temporal kernel induced by the Markov drift model. Extending constant-exploration guarantees to more general temporal kernels would clarify which forms of nonstationarity are sufficient to prevent posterior overconfidence. Another important direction is to remove the assumption that the drift parameter is known and fixed, allowing $\beta$ to be adapted online when the rate of temporal variation is unknown or changes over time. Finally, sharper concentration tools for the realized regret of TV-GP-UCB could improve the current high-probability bounds and better reflect the empirical stability observed in the simulations.

\begin{ack}
The first author gratefully acknowledges support from the NWO Spinoza Prize awarded to Aad van der Vaart.
\end{ack}

\bibliographystyle{plainnat}
\bibliography{sample}

\newpage

\appendix

\setcounter{theorem}{0}
\renewcommand{\thetheorem}{\thesection.\arabic{theorem}}
\renewcommand{\theHtheorem}{appendix.\thesection.\arabic{theorem}}

\section{Improved Expected Regret Bound}
\label{app:Improved_regret}

This appendix proves Theorem~\ref{thrm:improved_regret}. The argument proceeds in three steps: (i) bounding the instantaneous regret $r_t$ for Algorithm~\ref{alg:TV} conditionally on a high-probability event; (ii) bounding the expected regret on the complement event via concentration inequalities for the extrema of Gaussian processes; and (iii) summing the expected regret over $t$ and bounding the sum of predictive variances by the maximum information gain.

\subsection{Algorithm-specific bound}
\label{app:subsec:alg_specific}

\begin{lemma}[Instantaneous regret under constant $\beta$]
\label{lem:alg_spec_bnd}
Consider the setting of Theorem~\ref{thrm:improved_regret}. Fix $\delta\in(0,1)$ and $\tau\in\mathbb{N}^+$. Let $\mathcal{S}^*\subset\mathcal{S}$ be a discretization with $|\mathcal{S}^*|\le\tau^d$ satisfying
\begin{equation}
\label{eq:discretization_app}
\|x-[x]\|_1 \le \frac{rd}{\tau},\qquad \text{for all } x\in\mathcal{S},
\end{equation}
where $[x]\in\mathcal{S}^*$ is an $\ell^1$-closest grid point. Let $\beta$ and $C_2$ be defined as in Theorem~\ref{thrm:improved_regret}. If $x_t$ is selected according to Algorithm~\ref{alg:TV}, then for each fixed $t$, the instantaneous regret satisfies
\begin{equation}
\label{eq:inst_regret_app}
r_t = f_t(x_t^*)-f_t(x_t)\;\le\; 2\sqrt{\beta}\,\sigma_{t-1}(x_t) + C_2
\end{equation}
with probability at least $1-\delta$.
\end{lemma}

\begin{proof}
We bound the instantaneous regret by conditioning on the intersection of three high-probability events at a fixed time $t$. Let $\mathcal{E}_t = E_{1,t} \cap E_{2,t} \cap E_{3,t}$, defined as follows:
\begin{enumerate}[label=(\roman*), leftmargin=5ex]
    \item \textbf{Pointwise concentration:} Given the history $\mathcal{H}_{t-1}$, the query point $x_t$ is deterministic. Because $f_t(x_t) \sim \mathcal{N}(\mu_{t-1}(x_t), \sigma_{t-1}^2(x_t))$, standard Gaussian tail bounds ensure the event
    \begin{equation}
        |f_t(x_t)-\mu_{t-1}(x_t)| \le \sqrt{\beta}\,\sigma_{t-1}(x_t)
    \end{equation}
    holds with probability at least $1-\delta/3$ under the specified constant $\beta$.
    \item \textbf{Uniform grid concentration:} Applying a union bound over the $|\mathcal{S}^*| \le \tau^d$ elements of the discretization, the event
    \begin{equation}
        |f_t(x)-\mu_{t-1}(x)| \le \sqrt{\beta}\,\sigma_{t-1}(x), \quad \text{for all } x\in\mathcal{S}^*
    \end{equation}
    holds with probability at least $1-\delta/3$.
    \item \textbf{Lipschitz continuity:} By the derivative tail assumption \eqref{eq:smoothness} and a union bound over the $d$ spatial dimensions, setting $L = b\sqrt{\log(3ad/\delta)}$ ensures the event
    \begin{equation}
        |f_t(x)-f_t(x')| \le L\|x-x'\|_1, \quad \text{for all } x,x'\in\mathcal{S}
    \end{equation}
    holds with an unconditional probability of at least $1-\delta/3$ under the prior measure.
\end{enumerate}

Because events (i) and (ii) hold conditionally with probability $1-\delta/3$ for any history $\mathcal{H}_{t-1}$, the law of total expectation guarantees they also hold with an unconditional probability of at least $1-\delta/3$. Applying a union bound over the unconditional probabilities of these three events yields $\mathbb{P}(\mathcal{E}_t^c) \le \delta$, guaranteeing the joint event holds unconditionally with $\mathbb{P}(\mathcal{E}_t) \ge 1-\delta$. 
Conditioned on the realization of $\mathcal{E}_t$, the continuity event (iii) and the spatial discretization bound \eqref{eq:discretization_app} guarantee that the approximation error for any point in the domain is bounded by $L(rd/\tau) = C_2$.

Evaluating this spatial bound at the instantaneous optimum $x_t^*$ yields $f_t(x_t^*) \le f_t([x_t^*]) + C_2$. Because the projection $[x_t^*] \in \mathcal{S}^*$, we apply the uniform concentration event (ii) to bound $f_t([x_t^*])$, followed by the selection rule of Algorithm~\ref{alg:TV}, which maximizes the upper confidence bound at $x_t$:
\begin{align}
    f_t(x_t^*) &\le \mu_{t-1}([x_t^*]) + \sqrt{\beta}\,\sigma_{t-1}([x_t^*]) + C_2 \nonumber \\
    &\le \mu_{t-1}(x_t) + \sqrt{\beta}\,\sigma_{t-1}(x_t) + C_2. \label{eq:alg_spec_ucb_step}
\end{align}
Finally, applying the pointwise concentration event (i) at \(x_t\) yields
\[
\mu_{t-1}(x_t)
\le
f_t(x_t)+\sqrt{\beta}\,\sigma_{t-1}(x_t).
\]
Substituting this inequality into \eqref{eq:alg_spec_ucb_step} bounds the instantaneous regret:
\begin{equation}
    r_t = f_t(x_t^*) - f_t(x_t) \le 2\sqrt{\beta}\,\sigma_{t-1}(x_t) + C_2,
\end{equation}
which concludes the proof.
\end{proof}

\subsection{Algorithm-agnostic tail control}
\label{app:subsec:alg_agnostic}

The span \(W_t\) can be written as
\[
W_t=\sup_{x,y\in\mathcal S}\big(f_t(x)-f_t(y)\big).
\]
Thus \(W_t\) is the supremum of a centered Gaussian process indexed by
\(\mathcal S\times\mathcal S\). Since the sample paths are almost surely bounded on the compact domain and
\(k_S(x,x)\le 1\), Borell--TIS implies that \(W_t-\mathbb E[W_t]\) has a sub-Gaussian upper tail with a finite variance proxy. We denote this proxy by \(\sigma_w^2\).

\begin{lemma}[Tail expectation bound]
\label{lem:tail_expectation_span}
Assume \(W_t-\mu_w\) has the sub-Gaussian upper tail
\[
\mathbb P\{W_t-\mu_w>u\}
\le
\exp\!\left(-\frac{u^2}{2\sigma_w^2}\right)
\qquad
\text{for all }u>0.
\]
Then, for every event \(A\) with \(\mathbb P(A)\le \delta\),
\[
\mathbb E[W_t\mathbf 1_A]
\le
\delta\left(\mu_w+2\sigma_w\sqrt{2\log(e/\delta)}\right).
\]
\end{lemma}

\begin{proof}
Write \(X=(W_t-\mu_w)_+\). Then
\[
\mathbb E[W_t\mathbf 1_A]
\le
\mu_w\mathbb P(A)+\mathbb E[X\mathbf 1_A].
\]
Let \(q=\sigma_w\sqrt{2\log(e/\delta)}\). We split
\[
\mathbb E[X\mathbf 1_A]
\le
q\mathbb P(A)+\mathbb E[X\mathbf 1_{\{X>q\}}].
\]
Using the sub-Gaussian tail and integrating,
\[
\mathbb E[X\mathbf 1_{\{X>q\}}]
=
q\mathbb P(X>q)+\int_q^\infty \mathbb P(X>u)\,du.
\]
Since \(q=\sigma_w\sqrt{2\log(e/\delta)}\), we have
\[
\mathbb P(X>q)
\le
\exp\!\left(-\frac{q^2}{2\sigma_w^2}\right)
=
\frac{\delta}{e}.
\]
Moreover, the standard Gaussian tail integral bound gives
\[
\int_q^\infty \exp\!\left(-\frac{u^2}{2\sigma_w^2}\right)\,du
\le
\frac{\sigma_w^2}{q}\exp\!\left(-\frac{q^2}{2\sigma_w^2}\right)
=
\frac{\sigma_w^2}{q}\frac{\delta}{e}
\le
q\frac{\delta}{e},
\]
where the last inequality uses \(q\ge \sigma_w\). Therefore,
\[
\mathbb E[X\mathbf 1_{\{X>q\}}]
\le
2q\frac{\delta}{e}
\le
q\delta.
\]
Combining this with \(q\mathbb P(A)\le q\delta\) yields
\[
\mathbb E[X\mathbf 1_A]\le 2q\delta.
\]
\end{proof}

\subsection{Expected cumulative regret}
\label{app:subsec:sum_regret}

\begin{proof}[Proof of Theorem~\ref{thrm:improved_regret}]
Fix $\delta\in(0,1)$ and $\tau\in\mathbb{N}^+$, and define $\beta$ and $C_2$ as in Theorem~\ref{thrm:improved_regret}. By Lemma~\ref{lem:alg_spec_bnd}, for each time step $t$ there exists a confidence event $\mathcal{E}_t$ with $\mathbb{P}(\mathcal{E}_t)\ge 1-\delta$ such that, conditioned on $\mathcal{E}_t$, the instantaneous regret satisfies $r_t \le 2\sqrt{\beta}\,\sigma_{t-1}(x_t)+C_2$. 

Because instantaneous regret is always non-negative, we can bound it pathwise over the entire probability space by introducing the spatial span $W_t = \max_{x} f_t(x) - \min_{x} f_t(x)$ as the worst-case penalty when the confidence event fails:
\begin{equation}
\label{eq:pathwise_rt_bound}
r_t \;\le\; 2\sqrt{\beta}\,\sigma_{t-1}(x_t) + C_2 + W_t \mathbbm{1}_{\mathcal{E}_t^c},
\end{equation}
where $\mathbbm{1}_{\mathcal{E}_t^c}$ is the indicator for the complement event. Taking the unconditional expectation of both sides, and applying Lemma~\ref{lem:tail_expectation_span} to bound the expected tail contribution $\mathbb{E}[W_t \mathbbm{1}_{\mathcal{E}_t^c}] \le \delta C_3$, yields the valid deterministic bound:
\begin{equation}
\label{eq:Er_t_bound_app}
\mathbb{E}[r_t] \;\le\; 2\sqrt{\beta}\,\mathbb{E}[\sigma_{t-1}(x_t)] + C_2 + \delta C_3.
\end{equation}
Summing \eqref{eq:Er_t_bound_app} over $t=1,\dots,T$, we obtain:
\begin{equation}
\mathbb{E}[R_T] \;\le\; 2\sqrt{\beta}\,\mathbb{E}\Bigg[\sum_{t=1}^T \sigma_{t-1}(x_t)\Bigg] + T C_2 + T \delta C_3.
\end{equation}

To bound the sum of the posterior standard deviations, we apply the Cauchy-Schwarz inequality strictly inside the expectation:
\begin{equation}
\mathbb{E}\Bigg[\sum_{t=1}^T \sigma_{t-1}(x_t)\Bigg] \;\le\; \mathbb{E}\Bigg[\sqrt{T\sum_{t=1}^T \sigma_{t-1}^2(x_t)}\Bigg].
\end{equation}
We then bound the sum of predictive variances by the maximum information gain. For every realized query sequence $x_{1:T}$, the Schur-complement determinant identity gives
\begin{equation}
\frac{1}{2}\sum_{t=1}^T \log\big(1+\sigma^{-2}\sigma_{t-1}^2(x_t)\big)
=
\frac{1}{2}\log\det\big(I_T+\sigma^{-2}\tilde K_T(x_{1:T})\big)
\le
\tilde\gamma_T.
\end{equation}
Combining this pathwise identity with the variance-to-logarithm inequality from \citet[Lemma~5.4]{srinivas2012information} under the assumption that $k_S(x,x)\le 1$ yields almost surely
\begin{equation}
\sum_{t=1}^T \sigma_{t-1}^2(x_t)
\le
\frac{2}{\log(1+\sigma^{-2})}\,\tilde\gamma_T.
\end{equation}

Defining $C_1 = 8/\log(1+\sigma^{-2})$, we have $\sum_{t=1}^T 4\beta\,\sigma_{t-1}^2(x_t) \le C_1\beta\,\tilde\gamma_T$. Because $\tilde\gamma_T$ is defined as the maximum information gain over all possible sequences, it is a deterministic upper bound. The expectation of a deterministic bound is simply the bound itself. Substituting this into the cumulative sum yields the final result:
\begin{equation}
\mathbb{E}[R_T] \le \sqrt{C_1 T \beta\,\tilde\gamma_T} + T C_2 + T \delta C_3.
\end{equation}
\end{proof}

\noindent The explicit dependence of $\tilde{\gamma}_T$ on $T$ and $\epsilon$ for the time-decayed kernel is detailed in Appendix~\ref{app:max_info_gain}.

\section{Maximum Information Gain}
\label{app:max_info_gain}

\begin{theorem}[Time-varying maximum information gain]
\label{thm:TV_inf_gain_fixed}
Let $\tilde \gamma_T := \sup_{x_{1:T}\in\mathcal S^T} \tfrac12 \log\det(I_T+\sigma^{-2}\tilde K_T)$ denote the maximum information gain under the separable spatio-temporal kernel $\tilde k((x,t),(x',t')) = k_S(x,x')\,\rho^{|t-t'|}$ with $\rho=\sqrt{1-\epsilon}\in[0,1)$.
Assume $|k_S(x,x')|\le 1$ for all $x,x'\in\mathcal S$.
Then for any block length $\tilde N\in\mathbb N^+$,
\begin{equation}
\tilde\gamma_T
\;\le\;
\Big(\Big\lceil\tfrac{T}{\tilde N}\Big\rceil\Big)
\left(
\gamma_{\tilde N}
+
\frac{\sigma^{-2}\epsilon\,\tilde N\sqrt{\tilde N^2-1}}{2\sqrt{3}}
\sqrt{\gamma_{\tilde N}}
\right)
\;\le\;
\Big(\tfrac{T}{\tilde N}+1\Big)
\left(
\gamma_{\tilde N}
+
\frac{\sigma^{-2}\epsilon\,\tilde N\sqrt{\tilde N^2-1}}{2\sqrt{3}}
\sqrt{\gamma_{\tilde N}}
\right),
\label{eq:time_varying_info}
\end{equation}
where $\gamma_{\tilde N} := \sup_{x_{1:\tilde N}\in\mathcal S^{\tilde N}} \tfrac12 \log\det(I_{\tilde N}+\sigma^{-2}K_{\tilde N})$ is the standard time-invariant maximum information gain for $f\sim \mathcal{GP}(0,k_S)$.
\end{theorem}

\begin{proof}
\textbf{Step 1 (block decomposition).}
Partition the sequence $\{1,\dots,T\}$ into $m=\lceil T/\tilde N\rceil$ consecutive blocks, each of length at most $\tilde N$.
This partition is used only in the analysis; the GP posterior is neither reset nor truncated at block boundaries.
Let $(\mathbf f^{(i)},\mathbf y^{(i)})$ denote the latent function values and noisy observations within block $i$.
By the chain rule for mutual information,
\[
I(\mathbf f_T;\mathbf y_T)
=
\sum_{i=1}^m I(\mathbf f_T;\mathbf y^{(i)} \mid \mathbf y^{(<i)}).
\]
Because $\mathbf y^{(i)} = \mathbf f^{(i)} + \mathbf z^{(i)}$ with $\mathbf z^{(i)}$ drawn independently of $(\mathbf f_T,\mathbf y^{(<i)})$, the variables satisfy the Markov property $\mathbf y^{(i)} \perp\!\!\!\perp (\mathbf f_T\setminus \mathbf f^{(i)}) \mid \mathbf f^{(i)}$.
Consequently, $I(\mathbf f_T;\mathbf y^{(i)} \mid \mathbf y^{(<i)}) = I(\mathbf f^{(i)};\mathbf y^{(i)} \mid \mathbf y^{(<i)})$. 

Furthermore, because the additive Gaussian noise $\mathbf z^{(i)}$ is independent of the sequence history, the conditional entropy simplifies to $H(\mathbf{y}^{(i)} \mid \mathbf{f}^{(i)}, \mathbf{y}^{(<i)}) = H(\mathbf{z}^{(i)})$. While conditioning reduces entropy ($H(\mathbf{y}^{(i)} \mid \mathbf{y}^{(<i)}) \le H(\mathbf{y}^{(i)})$), it does not universally reduce mutual information. However, due to this independent noise structure, the explicit algebraic cancellation holds:
\begin{equation}
I(\mathbf{f}^{(i)};\mathbf{y}^{(i)} \mid \mathbf{y}^{(<i)}) = H(\mathbf{y}^{(i)} \mid \mathbf{y}^{(<i)}) - H(\mathbf{z}^{(i)}) \le H(\mathbf{y}^{(i)}) - H(\mathbf{z}^{(i)}) = I(\mathbf{f}^{(i)};\mathbf{y}^{(i)}).
\end{equation}
It follows that
\[
I(\mathbf f_T;\mathbf y_T) \le \sum_{i=1}^m I(\mathbf f^{(i)};\mathbf y^{(i)}).
\]
Taking the supremum over the choice of locations $x_{1:T}\in\mathcal S^T$ yields $\tilde\gamma_T \le \sum_{i=1}^m \tilde\gamma_{n_i} \le m\,\tilde\gamma_{\tilde N}$, where $n_i\le \tilde N$ is the size of block $i$ and the sequence $\tilde\gamma_n$ is monotonically non-decreasing. \newline

\noindent \textbf{Step 2 (one-block bound).}
Fix a block of length $\tilde N$.
Define its spatial kernel matrix as $K:=K_{\tilde N}$ and its temporal correlation matrix as $D:=D_{\tilde N}$, where $D_{ij}=\rho^{|i-j|}$.
The spatio-temporal covariance over this block is the Hadamard product $\tilde K := K\circ D$.
By the Schur product theorem, $\tilde K$ is positive semi-definite because both $K\succeq 0$ and $D\succeq 0$.
Let $\mathbf 1$ denote the $\tilde N\times\tilde N$ all-ones matrix and define the difference matrix $A := \tilde K - K = K\circ(D-\mathbf 1)$.
Let
\[
M := I_{\tilde N}+\sigma^{-2}K \succ 0, \qquad E := \sigma^{-2}A,
\]
so that $I_{\tilde N}+\sigma^{-2}\tilde K = M+E \succ 0$.
Applying the congruence factorization $M+E = M^{1/2}\big(I_{\tilde N} + B\big)M^{1/2}$ with $B := M^{-1/2}EM^{-1/2}$ yields
\[
\tfrac12\log\det(M+E)
=
\tfrac12\log\det(M) + \tfrac12\log\det(I_{\tilde N}+B).
\]
The matrix $B$ is symmetric and satisfies $I_{\tilde N}+B \succ 0$.
Therefore, every eigenvalue of $B$ is strictly greater than $-1$.
Using the scalar inequality $\log(1+u)\le u$ for all $u>-1$, we bound the log-determinant term by the trace:
\[
\log\det(I_{\tilde N}+B)
=
\sum_{j=1}^{\tilde N}\log\big(1+\lambda_j(B)\big)
\le
\sum_{j=1}^{\tilde N}\lambda_j(B)
=
\mathrm{tr}(B).
\]
Substituting this trace inequality provides the upper bound
\[
\tfrac12\log\det(M+E)
\le
\tfrac12\log\det(M) + \tfrac12\,\mathrm{tr}(B).
\]
By the cyclicity of the trace operator, $\mathrm{tr}(B)=\mathrm{tr}(M^{-1}E)$. Since $D_{ii}=1$, the diagonal entries of $A=K\circ(D-\mathbf 1)$ vanish, and hence $\mathrm{tr}(E)=0$. Therefore,
\[
\mathrm{tr}(B)
=
\mathrm{tr}\big((M^{-1}-I_{\tilde N})E\big).
\]
Using the Frobenius inner product,
\[
\mathrm{tr}(B)
\le
\left|\mathrm{tr}\big((M^{-1}-I_{\tilde N})E\big)\right|
\le
\|M^{-1}-I_{\tilde N}\|_F\,\|E\|_F.
\]

Write
\[
\Gamma(K)
:=
\frac12\log\det(I_{\tilde N}+\sigma^{-2}K).
\]
Let $\lambda_1,\ldots,\lambda_{\tilde N}$ be the eigenvalues of $K$ and set $u_j=\sigma^{-2}\lambda_j\ge 0$. Then
\[
\|M^{-1}-I_{\tilde N}\|_F^2
=
\sum_{j=1}^{\tilde N}
\left(\frac{u_j}{1+u_j}\right)^2
\le
\sum_{j=1}^{\tilde N}\log(1+u_j)
=
2\Gamma(K).
\]
Here the inequality follows from
\[
\left(\frac{u}{1+u}\right)^2\le\log(1+u),
\qquad u\ge0,
\]
since the difference has value zero at $u=0$ and derivative $(1+u^2)/(1+u)^3>0$. Hence
\[
\|M^{-1}-I_{\tilde N}\|_F
\le
\sqrt{2\Gamma(K)}.
\]

It remains to bound $\|E\|_F$. As above,
\[
|D_{ij}-1|
=
1-\rho^{|i-j|}
\le
\epsilon |i-j|.
\]
Since $|K_{ij}|\le1$,
\[
\|A\|_F^2
\le
\epsilon^2\sum_{i,j=1}^{\tilde N}(i-j)^2
=
\frac{\epsilon^2\tilde N^2(\tilde N^2-1)}{6},
\]
and therefore
\[
\|E\|_F
=
\sigma^{-2}\|A\|_F
\le
\frac{\sigma^{-2}\epsilon\,\tilde N\sqrt{\tilde N^2-1}}{\sqrt{6}}.
\]
Combining the two norm bounds gives
\[
\frac12\log\det(I_{\tilde N}+\sigma^{-2}\tilde K)
\le
\Gamma(K)
+
\frac{\sigma^{-2}\epsilon\,\tilde N\sqrt{\tilde N^2-1}}{2\sqrt{3}}
\sqrt{\Gamma(K)}.
\]
Since $s\mapsto s+c\sqrt{s}$ is increasing for $s\ge0$ and $\Gamma(K)\le\gamma_{\tilde N}$,
\[
\tilde\gamma_{\tilde N}
\le
\gamma_{\tilde N}
+
\frac{\sigma^{-2}\epsilon\,\tilde N\sqrt{\tilde N^2-1}}{2\sqrt{3}}
\sqrt{\gamma_{\tilde N}}.
\]
Summing this bound over the $m$ blocks in Step~1 completes the proof.
\end{proof}

\paragraph{Remark.}
The block decomposition follows the approach of \citet{bogunovic2016time}. For unit-spaced observations, an $\mathcal{O}(\epsilon\tilde N^{5/2})$ block perturbation also follows from the more general analysis of \citet{imamura2020time}. Without the zero-trace step above, the congruence argument recovers this order directly by combining $\|M^{-1}\|_F\le\sqrt{\tilde N}$ with $\|A\|_F\le\epsilon\tilde N^2$, without requiring sign restrictions on the perturbation. The additional identity $\mathrm{tr}(E)=0$ allows $M^{-1}$ to be replaced by $M^{-1}-I_{\tilde N}$, yielding the sharper bound in Theorem~\ref{thm:TV_inf_gain_fixed}.

\subsection{Asymptotic Scaling with Drift}

Theorem~\ref{thm:TV_inf_gain_fixed} gives a two-regime dependence on the drift rate once the block length is chosen according to the growth of the time-invariant maximum information gain.

\begin{corollary}[Information gain scaling]
\label{cor:info_gain_scaling}
Suppose the spatial kernel admits a time-invariant maximum information gain satisfying
\[
\gamma_{\tilde N}
=
\widetilde{\mathcal{O}}(\tilde N^q)
\]
for some capacity exponent $q\in[0,1)$. Define the polynomial block scale
\[
N_\epsilon
:=
\epsilon^{-\frac{2}{4-q}},
\]
and choose
\[
\tilde N
=
\min\left\{
T,\,
\left\lceil N_\epsilon\right\rceil
\right\}.
\]
Then
\begin{equation}
\tilde\gamma_T
=
\widetilde{\mathcal{O}}\left(
T^q
+
T\epsilon^{\frac{2(1-q)}{4-q}}
\right).
\label{eq:info_gain_two_regime}
\end{equation}
In particular, in the persistent-drift regime $T\gtrsim N_\epsilon$,
\begin{equation}
\tilde\gamma_T
=
\widetilde{\mathcal{O}}\left(
T\epsilon^{\frac{2(1-q)}{4-q}}
\right),
\label{eq:info_gain_persistent_drift}
\end{equation}
whereas for $T\lesssim N_\epsilon$ one has
\[
\tilde\gamma_T
=
\widetilde{\mathcal{O}}(T^q).
\]
For fixed $T$, choosing $\tilde N=T$ in Theorem~\ref{thm:TV_inf_gain_fixed} gives
\begin{equation}
\tilde\gamma_T
\le
\gamma_T
+
\frac{\sigma^{-2}\epsilon\,T\sqrt{T^2-1}}{2\sqrt{3}}
\sqrt{\gamma_T},
\label{eq:info_gain_static_branch}
\end{equation}
so the perturbation vanishes as $\epsilon\to0$ and the static information gain is recovered.
\end{corollary}

\begin{proof}
For $\tilde N\le T$, Theorem~\ref{thm:TV_inf_gain_fixed} and $\sqrt{\tilde N^2-1}\le\tilde N$ give
\[
\tilde\gamma_T
=
\widetilde{\mathcal{O}}\left(
\frac{T}{\tilde N}
\left(
\tilde N^q
+
\epsilon\tilde N^{2+q/2}
\right)
\right)
=
\widetilde{\mathcal{O}}\left(
T\tilde N^{q-1}
+
T\epsilon\tilde N^{1+q/2}
\right).
\]
Balancing the polynomial orders gives
\[
\tilde N^{q-1}
\asymp
\epsilon\tilde N^{1+q/2},
\qquad\text{or equivalently}\qquad
\tilde N
\asymp
\epsilon^{-\frac{2}{4-q}}
=
N_\epsilon.
\]
Thus, when $T\gtrsim N_\epsilon$, choosing $\tilde N$ at this scale gives
\[
\tilde\gamma_T
=
\widetilde{\mathcal{O}}\left(
T\epsilon^{\frac{2(1-q)}{4-q}}
\right).
\]

When $T\lesssim N_\epsilon$, choose $\tilde N=T$. Equation~\eqref{eq:info_gain_static_branch} and $\gamma_T=\widetilde{\mathcal{O}}(T^q)$ give
\[
\tilde\gamma_T
=
\widetilde{\mathcal{O}}\left(
T^q+\epsilon T^{2+q/2}
\right).
\]
The condition $T\lesssim N_\epsilon$ implies $\epsilon T^{2-q/2}\lesssim1$, so the perturbation term is at most $\widetilde{\mathcal{O}}(T^q)$. Combining the two regimes yields \eqref{eq:info_gain_two_regime}.
\end{proof}

For fixed $T$, the temporal matrix satisfies $D_{ij}\to1$ as $\epsilon\to0$. For the continuous kernels considered here on the compact domain $\mathcal S$, continuity of the log determinant under maximization therefore gives $\tilde\gamma_T\to\gamma_T$.

\section{Concentration of $R_T$}
\label{app:Convergence}

This appendix analyzes the concentration properties of the time-varying setting. Section~\ref{app:oracle_concentration} establishes the convergence of the oracle-value process via exponential covariance decay, which justifies the asymptotic average reward perspective discussed in Section~\ref{sec:reward_maximizing_perspective}. Section~\ref{app:span_concentration} proves the span-only concentration fallback, and Section~\ref{app:proof_realized_regret} proves the algorithm-specific realized-regret bound in Proposition~\ref{prop:realized_regret}.

\subsection{Concentration of the oracle-value process via covariance decay}
\label{app:oracle_concentration}

Define \( M_t := \sup_{x\in\mathcal S} f_t(x) \) for \( t\ge 1 \). Under the TV-GP dynamics, the sequence \(\{f_t\}\) is strictly stationary with each \( f_t\sim\mathcal{GP}(0,k_S) \). Consequently, \(\{M_t\}\) is stationary with a constant expectation \(\E[M_t]\). Assume the functions \(f_t\) have almost surely bounded sample paths on \(\mathcal S\) and are separable, implying \(M_t<\infty\) almost surely. By the Borell--TIS inequality \citep[Theorem~2.1.1]{adler2009random}, \(M_t\) is sub-Gaussian around its mean:
\[
\Pr\!\big(|M_t-\E[M_t]|>u\big)\le 2\exp\!\Big(-\frac{u^2}{2\sigma_M^2}\Big),
\qquad
\sigma_M^2:=\sup_{x\in\mathcal S}\Var\!\big(f_t(x)\big)\le 1.
\]
This concentration ensures \(\E[M_t^2]<\infty\). The following generic lemma controls the variance of a sum of random variables that exhibit exponentially decaying covariance.

\begin{lemma}[Averaging under exponential covariance decay]
\label{lem:covariance_decay_concise}
Let \(\{X_t\}_{t\ge 1}\) be random variables with \(\E[X_t^2]<\infty\) and suppose there exist \(m>0\) and \(c\in[0,1)\) such that \(\Cov(X_i,X_j)\le m\,c^{|i-j|}\) for all \(i,j\). Let \(S_T:=\sum_{t=1}^T X_t\). Then:
\begin{enumerate}[label=(\roman*)]
\item \(\Var(S_T)\le m\Big(T + 2\sum_{h=1}^{T-1}(T-h)c^h\Big)\le m\,T\,\frac{1+c}{1-c}\).
\item For any \(\eta>0\),
\[
\Pr\!\left(\left|\frac{S_T}{T}-\frac{\E[S_T]}{T}\right|>\eta\right)\le \frac{m}{T\,\eta^2}\cdot\frac{1+c}{1-c}\xrightarrow[T\to\infty]{}0,
\]
hence \(\frac{S_T}{T}-\frac{\E[S_T]}{T}\to 0\) in probability.
\item For any \(\omega\in(0,1)\),
\[
S_T \le \E[S_T]+\sqrt{\Var(S_T)/\omega} \ \ \le\ \ \E[S_T]+\sqrt{\frac{mT}{\omega}\cdot\frac{1+c}{1-c}} \qquad \text{with probability at least }1-\omega .
\]
\end{enumerate}
\end{lemma}

\begin{proof}
By bilinearity of covariance, grouping terms by lag \(h=|i-j|\) yields
\[
\Var(S_T)=\sum_{i=1}^T\sum_{j=1}^T \Cov(X_i,X_j) \le m\sum_{i=1}^T\sum_{j=1}^T c^{|i-j|} = m\Big(T+2\sum_{h=1}^{T-1}(T-h)c^h\Big).
\]
Since \((T-h)\le T\) and \(\sum_{h=1}^\infty c^h = c/(1-c)\), we obtain
\[
T+2\sum_{h=1}^{T-1}(T-h)c^h \le T+2T\sum_{h=1}^\infty c^h = T\frac{1+c}{1-c},
\]
proving (i). Items (ii) and (iii) follow by Chebyshev's inequality applied to \(S_T/T\) and \(S_T\), respectively.
\end{proof}

The next lemma bounds the covariance of suprema under the canonical Gaussian autoregressive coupling that underlies the TV-GP recursion.

\begin{lemma}[Covariance contraction for Gaussian suprema]
\label{lem:cov_sup_contraction}
Let \(f\) and \(g\) be centered, separable Gaussian processes indexed by a set \(T\), with \(g(t)=\rho f(t)+\sqrt{1-\rho^2}\,h(t)\) for \(\rho\in[0,1]\), where \(h\) is an independent copy of \(f\). Assume \(\sup_{t\in T}\Var(f(t))\le \sigma_T^2<\infty\) and \(\sup_{t\in T}|f(t)|<\infty\) almost surely. Then \(\Cov\big(\sup_{t\in T} f(t), \sup_{t\in T} g(t)\big)\le \rho\,\sigma_T^2\).
\end{lemma}

\begin{proof}
We first prove the claim for a finite set \(T=\{1,\dots,k\}\) and then pass to a general separable \(T\). Let \(A=(f(1),\dots,f(k))\) and \(C=(g(1),\dots,g(k))\). By construction, the joint distribution of \((A,C)\) is a zero-mean Gaussian with cross-covariance \(\Cov(A_i,C_j)=\rho\,\Sigma_{ij}\), where \(\Sigma_{ij} = \Cov(A_i,A_j)\).

Let \(\phi_n(u) := \frac{1}{n}\log\big(\sum_{i=1}^k e^{n u_i}\big)\) be the smooth approximation of the maximum function. The gradient \(\nabla\phi_n(u)\) is the softmax function, ensuring that for any \(u\), \(\nabla\phi_n(u)\) is a probability vector (\(\|\nabla\phi_n(u)\|_1=1\) with strictly non-negative components).

To compute the covariance \(\Cov(\phi_n(A),\phi_n(C))\), we employ the Gaussian interpolation technique standard in the proofs of Gaussian comparison inequalities \citep[see, e.g.,][Section 2.2]{adler2009random}. Consider the interpolation path parameterized by \(t \in [0,1]\) with the block covariance matrix
\[
\Gamma(t) = \begin{pmatrix} \Sigma & t \rho \Sigma \\ t \rho \Sigma & \Sigma \end{pmatrix}.
\]
Let \((A(t), C(t)) \sim \mathcal{N}(0, \Gamma(t))\). For \(t=0\), \(A(0)\) and \(C(0)\) are independent, implying \(\Cov(\phi_n(A(0)), \phi_n(C(0))) = 0\). For \(t=1\), we recover the joint distribution of \((A,C)\). Differentiating the expectation of \(\phi_n(A(t))\phi_n(C(t))\) with respect to \(t\) and applying multivariate Gaussian integration by parts \citep[Lemma 2.1.4]{adler2009random} yields:
\[
\frac{d}{dt}\E[\phi_n(A(t))\phi_n(C(t))] = \sum_{i=1}^k \sum_{j=1}^k \frac{d \Gamma_{AC,ij}(t)}{dt} \E\!\left[ \frac{\partial \phi_n(A(t))}{\partial A_i} \frac{\partial \phi_n(C(t))}{\partial C_j} \right].
\]
Since \(\frac{d \Gamma_{AC,ij}(t)}{dt} = \rho \Sigma_{ij}\), integrating from \(t=0\) to \(t=1\) establishes the exact identity:
\[
\Cov(\phi_n(A),\phi_n(C)) = \rho \int_0^1 \E\!\left[ \langle \nabla \phi_n(A(t)), \Sigma \nabla \phi_n(C(t)) \rangle \right] dt.
\]

We bound the bilinear form inside the expectation via \(\ell_1/\ell_\infty\) duality. Let \(p = \nabla \phi_n(A(t))\) and \(q = \nabla \phi_n(C(t))\). Because \(p\) and \(q\) are probability vectors,
\[
\sum_{i,j} p_i \Sigma_{ij} q_j \le \left( \max_{i,j} |\Sigma_{ij}| \right) \sum_{i,j} p_i q_j = \max_{i,j} |\Sigma_{ij}|.
\]
Because \(\Sigma\) is positive semi-definite, its maximum absolute entry lies on the diagonal: \(\max_{i,j} |\Sigma_{ij}| \le \max_i \Sigma_{ii}\). By assumption, \(\max_i \Sigma_{ii} \le \sigma_T^2\). Therefore, the bilinear form is bounded almost surely by \(\sigma_T^2\), independent of the dimension \(k\).

Substituting this bound into the integral yields:
\[
\Cov(\phi_n(A),\phi_n(C)) \le \rho \int_0^1 \E[\sigma_T^2] dt = \rho\sigma_T^2.
\]

The approximations \(\phi_n(A)\) and \(\phi_n(C)\) converge almost surely to \(\phi(A)\) and \(\phi(C)\). The Borell--TIS inequality ensures that \(\{\phi_n(A)\}\) and \(\{\phi_n(C)\}\) are uniformly square-integrable. Vitali's convergence theorem permits passing the limit through the covariance, which yields \(\Cov(\max_i A_i, \max_i C_i)\le \rho\sigma_T^2\).

To extend this to a general separable set \(T\), let \(\{T_m\}_{m\ge 1}\) be an increasing sequence of finite subsets whose union is dense in a separability set for \(f\) and \(g\). This guarantees that \(\sup_{t\in T_m} f(t)\) and \(\sup_{t\in T_m} g(t)\) converge almost surely to their respective suprema over \(T\). Applying the finite-set result yields \(\Cov(\sup_{t\in T_m} f(t), \sup_{t\in T_m} g(t))\le \rho\sigma_T^2\) for all \(m\). The Borell--TIS inequality provides \(\sup_{m}\E[(\sup_{t\in T_m} f(t))^2]<\infty\). Uniform integrability allows taking the limit as \(m\to\infty\), which establishes the bound for the general separable set \(T\).
\end{proof}

In the TV-GP model, the transition rule is \(f_{t+h}(\cdot)=\rho^h f_t(\cdot)+\sqrt{1-\rho^{2h}}\, \tilde g_{t,h}(\cdot)\), where \(\rho:=\sqrt{1-\epsilon}\) and \(\tilde g_{t,h}\) is a centered Gaussian process with the same spatial kernel as \(f_t\). The process \(\tilde g_{t,h}\) is independent of \(f_t\) because it is a linear combination of the independent innovations \(\{g_{t+1},\dots,g_{t+h}\}\). Setting \(T=\mathcal S\) in Lemma~\ref{lem:cov_sup_contraction} gives \(\sigma_T^2\le 1\) and \(\Cov(M_t,M_{t+h}) \le \rho^h\) for \(h\ge 0\). We apply Lemma~\ref{lem:covariance_decay_concise} by substituting \(X_t:=M_t\), \(m=1\), and \(c=\rho\). For any \(\eta>0\),
\[
\Pr\!\left(\left|\frac1T\sum_{t=1}^T M_t-\E[M_1]\right|>\eta\right) \le \Pr\!\left(\left|\frac1T\sum_{t=1}^T M_t-\frac{\E[\sum_{t=1}^T M_t]}{T}\right|>\eta\right) \to 0,
\]
since \(\E[\sum_{t=1}^T M_t]/T=\E[M_1]\) by stationarity. Equivalently, the time average \(\frac1T\sum_{t=1}^T f_t(x_t^\star)\) converges in probability to \(\E[\sup_{x\in\mathcal S} f_1(x)]\). The deviation bound in Lemma~\ref{lem:covariance_decay_concise}(iii) provides an explicit \(\mathcal{O}(\sqrt{T})\) concentration inequality, with constants depending on \(\epsilon\) entirely through the factor \((1+\rho)/(1-\rho)\).

\subsection{Span-only concentration fallback}
\label{app:span_concentration}

This subsection justifies the span-only fallback bound used in Section~\ref{subsec:prob_regret_bound}. Recall that
\[
W_t=\sup_{x,y\in\mathcal S}\big(f_t(x)-f_t(y)\big).
\]
Since \(r_t\le W_t\) for every policy and every round,
\[
R_T\le \sum_{t=1}^T W_t.
\]
Define
\[
S_T^W:=\sum_{t=1}^T W_t.
\]
Equivalently,
\[
S_T^W
=
\sup_{(x_1,y_1),\ldots,(x_T,y_T)\in\mathcal S\times\mathcal S}
\sum_{t=1}^T \big(f_t(x_t)-f_t(y_t)\big).
\]
Thus \(S_T^W\) is the supremum of a centered Gaussian process indexed by
\((\mathcal S\times\mathcal S)^T\).

Let \(\rho=\sqrt{1-\epsilon}\). For any fixed index sequence
\((x_1,y_1),\ldots,(x_T,y_T)\), the variance of the corresponding Gaussian variable satisfies
\[
\begin{aligned}
\operatorname{Var}\!\left(
\sum_{t=1}^T \big(f_t(x_t)-f_t(y_t)\big)
\right)
&=
\sum_{t,s=1}^T
\rho^{|t-s|}
\Big[
k_S(x_t,x_s)-k_S(x_t,y_s)-k_S(y_t,x_s)+k_S(y_t,y_s)
\Big] \\
&\le
4\sum_{t,s=1}^T \rho^{|t-s|}  \\
&\le
4T\frac{1+\rho}{1-\rho}.
\end{aligned}
\]
Here we used \(k_S(x,x)\le 1\), which implies
\[
\operatorname{Var}\big(f(x)-f(y)\big)\le 4
\]
and bounds the absolute covariance between two such increments by \(4\).

By Borell--TIS, for all \(u>0\),
\[
\mathbb P\!\left\{
S_T^W-\mathbb E[S_T^W]>u
\right\}
\le
\exp\!\left(
-\frac{u^2}{8T(1+\rho)/(1-\rho)}
\right).
\]
Stationarity gives \(\mathbb E[S_T^W]=T\mu_w\). Therefore, with probability at least \(1-\omega\),
\[
S_T^W
\le
T\mu_w
+
\sqrt{
8T\frac{1+\rho}{1-\rho}\log(1/\omega)
}.
\]
Since \(R_T\le S_T^W\), we obtain
\[
R_T
\le
T\mu_w
+
\sqrt{
8T\frac{1+\rho}{1-\rho}\log(1/\omega)
}
\]
with probability at least \(1-\omega\).

\subsection{Concentration of realized regret (Proof of Proposition~\ref{prop:realized_regret})}
\label{app:proof_realized_regret}

Lemma~\ref{lem:alg_spec_bnd} establishes that the combined confidence and regularity event $\mathcal{E}_t$ holds with an unconditional probability of $\mathbb{P}(\mathcal{E}_t) \ge 1-\delta$. Because all instantaneous regret terms are non-negative, $r_t$ is bounded pathwise by the UCB width on the success event and the spatial span on the failure event:
\begin{equation}
r_t \;\le\; 2\sqrt{\beta}\sigma_{t-1}(x_t) + C_2 + W_t \mathbb{I}_{\mathcal{E}_t^c},
\end{equation}
where $W_t = \sup_{x\in\mathcal{S}} f_t(x) - \inf_{x\in\mathcal{S}} f_t(x)$. Summing over $T$ and applying the Cauchy-Schwarz inequality to the standard deviations yields the realized cumulative regret:
\begin{equation}
R_T \;\le\; \underbrace{\sqrt{C_1 T \beta \tilde{\gamma}_T} \;+\; T C_2}_{:= U_T} \;+\; \underbrace{\sum_{t=1}^T W_t \mathbb{I}_{\mathcal{E}_t^c}}_{:= F_T}.
\label{eq:pathwise_sum_appendix}
\end{equation}
Since $\tilde{\gamma}_T$ deterministically bounds the information gain for any sequence of length $T$, $U_T$ is a valid upper bound almost surely. 

To bound the stochastic failure penalty $F_T$, we rely on the unconditional properties of the spatial span. By Lemma~\ref{lem:tail_expectation_span}, the expected product of the sub-Gaussian span and the failure indicator satisfies
\[
\mathbb{E}[W_t \mathbb{I}_{\mathcal{E}_t^c}] \le \delta C_3.
\]

Because instantaneous regret is strictly non-negative, $F_T \ge 0$. Applying Markov's inequality directly to the sum yields:
\begin{equation}
\mathbb{P}\left( F_T > \frac{T \delta C_3}{\omega} \right) \le \frac{\mathbb{E}[F_T]}{(T \delta C_3) / \omega} = \omega.
\end{equation}
Substituting this limit into \eqref{eq:pathwise_sum_appendix} yields the bound in \eqref{eq:bound_markov}.

\section{Kalman filtering implementation for the simulation study}
\label{app:kalman_filter_simulations}

This appendix describes the exact posterior updates used in Section~\ref{sec:simulation_study}. On the finite simulation grid $\{x^{(1)},\ldots,x^{(N)}\}$, define $f_t=(f_t(x^{(1)}),\ldots,f_t(x^{(N)}))^\top\in\mathbb{R}^N$, and let $K_S\in\mathbb{R}^{N\times N}$ be the spatial covariance matrix with entries $(K_S)_{ij}=k_S(x^{(i)},x^{(j)})$. Restricting the TV-GP recursion \eqref{eq:tvgp} to this grid gives the linear Gaussian state-space model
\[
\begin{aligned}
f_t &= \rho f_{t-1}+\xi_t,
&\qquad
\xi_t &\sim \mathcal{N}(0,\epsilon K_S), \\
y_t &= H_t f_t+v_t,
&\qquad
v_t &\sim \mathcal{N}(0,\sigma^2),
\end{aligned}
\]
where $\rho=\sqrt{1-\epsilon}$ and $H_t\in\mathbb{R}^{1\times N}$ selects the coordinate corresponding to the sampled grid point $x_t$.

Let $\mu_{t|t}$ and $P_{t|t}$ denote the posterior mean and covariance of $f_t$ after observing $y_t$, and let $\mu_{t|t-1}$ and $P_{t|t-1}$ denote the one-step predictive quantities. We initialize $\mu_{1|0}=0$ and $P_{1|0}=K_S$. Exact inference is performed by the Kalman recursion
\[
\begin{aligned}
\mu_{t|t-1} &= \rho\mu_{t-1|t-1},
&
P_{t|t-1} &= \rho^2 P_{t-1|t-1}+\epsilon K_S, \\
S_t &= H_tP_{t|t-1}H_t^\top+\sigma^2,
&
L_t &= P_{t|t-1}H_t^\top S_t^{-1}, \\
\mu_{t|t} &= \mu_{t|t-1}+L_t(y_t-H_t\mu_{t|t-1}),
&
P_{t|t} &= P_{t|t-1}-L_tH_tP_{t|t-1}.
\end{aligned}
\]
The acquisition rule is evaluated on the grid using the predictive mean and marginal variances,
\[
i_t\in\arg\max_{i\in\{1,\ldots,N\}}
\left\{
(\mu_{t|t-1})_i+\sqrt{\beta}\sqrt{(P_{t|t-1})_{ii}}
\right\}.
\]
This recursion is equivalent to GP posterior inference under the spatio-temporal kernel \eqref{eq:sep_kernel} restricted to the grid, but avoids recomputing the posterior from the full observation history. Since $N$ is fixed, each posterior update costs $\mathcal{O}(N^2)$ rather than growing with $t$, enabling exact simulations over long horizons.

\section{Simulation reproducibility details}
\label{app:simulation_reproducibility}

This appendix provides the reproducibility details for the simulation study in Section~\ref{sec:simulation_study}. The experiments use synthetic data generated directly from the TV-GP model in Equation~\eqref{eq:tvgp}. No external datasets, pretrained models, or human-subject data are used.

The action space is the fixed $20\times20$ grid on $\mathcal S=[0,1]^2$, giving $N=400$ candidate actions. The spatial kernel is the squared exponential kernel with lengthscale $l=0.1$ and prior variance $1.0$. The observation noise variance is fixed at $\sigma^2=0.01$. Each reported average regret value is computed over $100$ independent TV-GP sample paths of length $T=10{,}000$.

For each drift rate $\epsilon$ and exploration parameter $\beta$, the latent process is simulated according to the recursion in Equation~\eqref{eq:tvgp}. Posterior inference is performed exactly on the finite grid using the Kalman filtering recursion in Appendix~\ref{app:kalman_filter_simulations}. The constant-$\beta$ TV-GP-UCB policy selects actions according to
\[
i_t\in\arg\max_{i\in\{1,\ldots,N\}}
\left\{
(\mu_{t|t-1})_i+\sqrt{\beta}\sqrt{(P_{t|t-1})_{ii}}
\right\}.
\]
The logarithmic baseline uses the schedule
\[
\beta_t=c_1\log(c_2t),
\qquad c_1=0.8,\quad c_2=4,
\]
as described in Section~\ref{sec:simulation_study}.

The accompanying code archive contains the scripts used to generate the synthetic sample paths, run constant-$\beta$ TV-GP-UCB, run the logarithmic time-varying-$\beta_t$ baseline, compute the average regret curves, bootstrap the empirical minimizers $\hat\beta_{\mathrm{opt}}$, and regenerate Figure~\ref{fig:simulation-study}.

For the left panel of Figure~\ref{fig:simulation-study}, shaded bands denote approximate 95\% confidence intervals for the mean average regret, computed as $\pm 1.96$ standard errors over the $100$ independent trials. For the right panel, the uncertainty intervals for $\hat\beta_{\mathrm{opt}}$ are computed using bootstrap percentile intervals over the independent trials.

The computations for Figure~\ref{fig:simulation-study} were run on a MacBook Pro with an Apple M1 Pro chip, 8 CPU cores, and 16 GB of unified memory. The simulations used CPU multiprocessing and no GPU acceleration. The complete run used to generate the final figure took approximately 48 hours of wall-clock time. Memory usage was modest because the Kalman filter operates on a fixed $N=400$ state dimension, but the total runtime is dominated by the large number of independent trials and parameter settings.

\end{document}